\documentclass[letterpaper, 10 pt, conference]{ieeeconf}  

\IEEEoverridecommandlockouts                              

\usepackage{cite}
\usepackage{amsmath,amssymb}
\usepackage{algorithm}
\usepackage{threeparttable}
\usepackage{algorithmic}
\renewcommand{\algorithmicrequire}{ \textbf{Input:}} 
\renewcommand{\algorithmiccomment}{ \textbf{Initialize:}} 
\renewcommand{\algorithmicensure}{ \textbf{Result:}} 

\usepackage{graphicx}
\usepackage{subfigure}
\usepackage{textcomp}
\usepackage{multirow}
\usepackage{xcolor}
\usepackage{bm,color}
\usepackage{cuted}
\usepackage[citecolor=blue,urlcolor=blue]{hyperref}
\hypersetup{
		colorlinks=true,
		linkcolor=blue,
		filecolor=red,      
		urlcolor=blue,
		citecolor=blue,
	}

\usepackage{graphicx} 
\usepackage{subfigure}
\usepackage{colortbl}
\usepackage{wrapfig}

\usepackage{soul}

\newtheorem{thm}{Theorem}

\newtheorem{lem}{Lemma}
\newtheorem{rem}{Remark}
\newtheorem{cor}{Corollary}

\title{\LARGE \bf
Learning-based Observer for Coupled Disturbance
}

	\author{Jindou Jia$^{1,2,*}$, Meng Wang$^{2, *}$, Zihan Yang$^{2}$, Bin Yang$^{2}$, Yuhang Liu$^{2}$, Kexin Guo$^{2,\dagger}$, and Xiang Yu$^{2}$
    \thanks{*Equal contribution.}
    \thanks{$\dagger$Corresponding author(kxguo@buaa.edu.cn).}
	\thanks{
		$^{1}$Nanyang Technological University, 639798, Singapore. $^{2}$Beihang University, 100191, Beijing, China.}
	}

\begin{document}

\maketitle
\thispagestyle{empty}
\pagestyle{empty}

\begin{abstract}

 Achieving high-precision control for robotic systems is hindered by the low-fidelity dynamical model and external disturbances. Especially, the intricate coupling between internal uncertainties and external disturbances further exacerbates this challenge. This study introduces an effective and convergent algorithm enabling accurate estimation of the coupled disturbance via combining control and learning philosophies. Concretely, by resorting to \textit{Chebyshev} series expansion, the coupled disturbance is effectively decomposed into an unknown parameter matrix and two known structures dependent on system state and external disturbance respectively. A regularized least squares process is subsequently formalized to learn the parameter matrix using historical time-series data. Furthermore, a polynomial disturbance observer is specifically devised to achieve a high-precision estimation of the coupled disturbance by utilizing the learned {structure} portion. Extensive simulations and real flight tests valid the effectiveness of the proposed framework. We believe this work can offer a new pathway to integrate learning approaches into control frameworks for addressing longstanding challenges in robotic applications.

\end{abstract}

 \section{Introduction}
High-precision control is crucial for robotic systems where {states-related} model uncertainties and external factors-conditioned disturbances are pervasive. A plethora of estimation schemes have been developed regarding model uncertainties and external disturbances as a lumped disturbance \cite{o2022neural,richards2022control, Neural_Networks}. \textcolor{black}{Few studies have penetrated the difficult-to-model coupling between internal states and external factors to reduce control conservativeness.}

In the control \textcolor{black}{area}, \textcolor{black}{focus on the coupled disturbance (depending not only on external disturbance but also internal model uncertainties),} many classical methods attempt to estimate or attenuate it with a bounded (or bounded derivative) assumption, such as $\mathcal{L}_1$ adaptive observer \cite{hovakimyan2010l1, huang2023datt}, extended state observer (ESO) \cite{ESO}, nonlinear disturbance observer (NDO) \cite{guo2005disturbance, NDO_bounded_variation}, composite adaptive control \cite{o2022neural, slotine1989composite}, sliding mode control \cite{yang1999sliding}, and $\mathcal{H}_{\infty}$ control \cite{khalil1996robust}. The bounded assumption on coupled disturbances has explicit limitations from a theoretical perspective. \textcolor{black}{A causality dilemma arises that the state should be preemptively restricted within a bound whose upper limit may be unreasonable, before proving the stability \cite{10100656}.} Moreover, this assumption usually results in a bounded estimation error \cite{guo2005disturbance}. A smaller bound of the coupled disturbance (or its derivative) is imposed to enable a small estimation/control error, which may not be satisfied. 



Benefiting from the growing computing power, massively available data, and advanced algorithms, data-driven learning approaches appear to be an alternative for handling robotic disturbances. In the data-driven paradigm, most approaches attempt to learn the unknown lumped disturbance with neural networks (NNs) \cite{Neural_Networks, saviolo2023learning}. \textcolor{black}{Focus on the coupled disturbance}, a major challenge emerges that the external factor-conditioned disturbances cannot be entirely sampled as learning input (e.g., the real wind speed for quadrotors). Inspired by the time-delay embedding {theorem}\cite{takens2006detecting}, historical states are concentrated as the learning input to reflect unmeasurable factors \cite{haggerty2023control, lale2024falcon}. With the assistance of meta-learning, several works establish a bi-level optimization to handle coupled wind disturbance \cite{o2022neural, richards2022control, xie2023hierarchical}. {By merging with} composite adaptive control \cite{slotine1989composite}, the meta-learning scheme can remarkably improve the {dynamic} performance, but this scheme is limited to specific disturbance and has yet to result in a zero-error estimation theoretically. Moreover, training {process} for these methods is black-box and labor-intensive.

\begin{figure}
	\centering
	\includegraphics[width=0.8\linewidth]{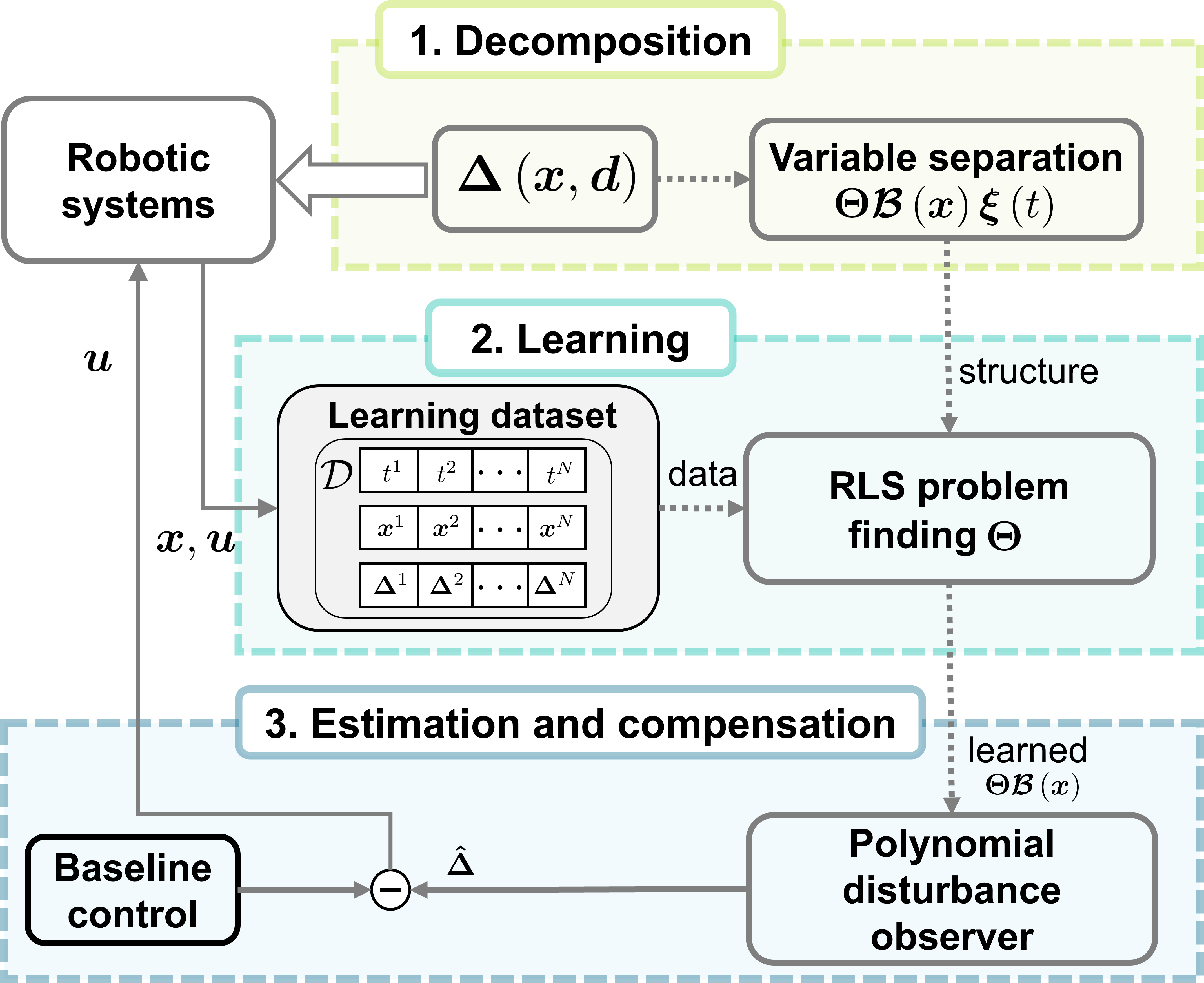}
	\caption{\textcolor{black}{Overall framework of our proposed learning-based observer.} Supplementary Materials: \url{https://anonymous.4open.science/r/Coupled_disturbance-33DF}.}
	\label{overall_framework}
\end{figure}

More recently, the interest in combining control-theoretic approaches with data-driven learning techniques is thriving for achieving stable and high-precision control performance \cite{o2022neural, richards2022control, huang2023datt, dawson2023safe, lutter2023combining, banderchuk2023combining, jia2025, jia2024feedback}. For more details refer to Supplementary Materials \ref{Related_Work}. Following this routine, we present a convergent estimation {framework} specifically for the coupled disturbance. Concretely, a learning algorithm is employed to learn the latent invariable parameters offline, while a polynomial observer is used to estimate the time-varying part online by utilizing the learned portion. As depicted in Fig. \ref{overall_framework}, the main contributions are summarized as follows. \textbf{Decomposition:} A variable decomposition principle (Theorem \ref{thm_separation}) is established to decompose the coupled disturbance into an unknown parameter matrix, a system-state-related matrix, and an external-disturbance-related vector, with an arbitrarily small residual. \textcolor{black}{The \textit{Chebyshev} polynomial is utilized due to its suitable property.} \textbf{Learning:} \textcolor{black}{With an analytically derivable assumption on external disturbance,} a corollary (Corollary \ref{Corollary_seperation}) is further developed, which enables the unknown parameter matrix to be learned in a supervised way. Afterward, the learning objective is formalized as a regularized least squares (RLS) problem with a closed-form solution. \textbf{Estimation and compensation:} By leveraging the learned knowledge, a polynomial disturbance observer is designed to estimate the final coupled disturbance (Theorem \ref{thm_DO_high}), which can improve control performance after compensation.

\textcolor{black}{In the proposed framework, 1) there is no need to manually model complex disturbance, 2) the bounded assumption on the coupled disturbance (or its derivative) in traditional control field \cite{o2022neural, hovakimyan2010l1, ESO, NDO_bounded_variation, yang1999sliding, khalil1996robust} can be avoided, and 3) the implemented learning strategy is explainable and lightweight compared with deep neural networks (DNNs) based methods \cite{o2022neural, richards2022control, Neural_Networks, saviolo2023learning}.} Multiple simulations and flight tests (indoor and outdoor) are implemented to verify the proposed scheme.

\subsection{Related Work} \label{Related_Work}
As our proposed method falls into the realm of the scheme combining control and data-driven learning, the related advanced research is also reviewed. 

\subsubsection{Analytical disturbance estimation}

The basic idea of the disturbance estimation approach is to design an ad-hoc observer to estimate the disturbance by utilizing its influence on the system. The estimation method is a two-degree-of-freedom control framework \cite{guo2005disturbance, disturbance_prediction}, which can achieve tracking and anti-disturbance performances simultaneously. For most disturbance observers, like {frequency domain disturbance observer} \cite{Frequency_DO}, ESO \cite{ESO}, unknown input observer (UIO) \cite{UIO}, generalized proportional integral observer (GPIO) \cite{GPIO}, and NDO \cite{NDO_bounded_variation}, zero-error estimation can be usually achieved in the event of constant disturbances. For more complicated time-varying disturbances, accurate estimation usually requires \textit{a priori} knowledge of disturbance features. For example, UIO \cite{UIO} and NDO \cite{NDO_harmonic} can accurately estimate the harmonic disturbance if its frequency is known. GPIO \cite{GPIO} and higher-order NDO \cite{high_order_NDO} can achieve an asymptotic estimation of the disturbance represented by a high-order polynomial of time series. More recently, for multi-disturbance with limited priori information, the simultaneous attenuation and compensation approach appears to be a nascent solution \cite{guo2005disturbance}.

Some researchers attempt to estimate the unknown coupled disturbance with a bounded derivative assumption, such as ESO \cite{ESO} and NDO \cite{NDO_bounded_variation}. The strict {assumption} presents an inherent theoretical limitation, as it requires the system state to be preemptively bounded \cite{10100656} ({Detailed expression is provide in Section \ref{Problem_Formulation}}). Moreover, a large derivative bound can result in a large estimation error. In order to avoid the strict {premise} on system states, a two-stage active disturbance rejection control strategy is designed in \cite{10100656}. The controller in the first stage guarantees the boundness of the system state by a special auxiliary function, while a linear ESO in the second stage is employed to estimate the total disturbance. However, the existence of the auxiliary function is not discussed. 


\subsubsection{Combining analytical control and data-driven learning}

Fueled by the advancement of updated learning algorithms, the interest in combining control-theoretic approaches with data-driven learning techniques is thriving for achieving stable and high-precision control performance. In \cite{dawson2023safe}, DNNs are utilized to synthesize control certificates such as \textit{Lyapunov} functions and barrier functions, {providing formal guarantee of the safety and stability.}  Moreover, the mass matrix and the potential energy in \textit{Lagrangian} mechanics and \textit{Hamiltonian} mechanics are learned by DNNs \cite{lutter2023combining}. Compared to naive black-box model learning, {a more interpretable and plausible model can be obtained from energy conservation principle.}. Furthermore, an interesting conclusion is found that a higher-order nonlinear system controller by the reinforcement learning (RL) policy behaves like a linear system \cite{li2022bridging}. The stability of the RL policy can be analyzed by the identified linear closed-loop system with the pole-zero method. \cite{banderchuk2023combining} combine robust control and echo state networks (ESN) to control nonlinear systems, where ESN is employed to learn the inverse dynamics. Nevertheless, the bound of disturbance and learning output need to be known in advance. \cite{huang2023datt} develop a $\mathcal{L}_1$ adaptive observer-argument RL, enabling the quadrotor to track infeasible trajectories in the presence of disturbances. 


Even with these advances, for nonlinear systems perturbed by external time-varying disturbances that cannot be accurately sampled, data-driven supervised learning methods would no longer be applicable. Several studies addressing the handling of coupling disturbances have been conducted by formulating a bi-level optimization problem \cite{o2022neural, richards2022control}. Within the framework of adaptive control, the nonlinear features depending on the system state are learned via meta-learning offline in \cite{o2022neural}. This work breaks through the assumption that the unknown dynamics are linearly parameterizable in the traditional adaptive control method. \cite{richards2022control} develop a control-oriented meta-learning framework, which uses the adaptive controller as the base learner to attune learning to the downstream control objective. Both methods attribute the effect of external disturbance in the last layer of the neural networks, which is estimated adaptively online. However, the above scheme ensures zero-error convergence only when the external disturbance is constant. Moreover, laborious offline training is imposed.

\textcolor{black}{In the proposed framework, 1) there is no need to manually model complex disturbance, 2) the bounded assumption on the coupled disturbance (or its derivative) in traditional control field \cite{o2022neural, hovakimyan2010l1, ESO, NDO_bounded_variation, yang1999sliding, khalil1996robust} can be avoided, and 3) the implemented learning strategy is explainable and lightweight compared with deep neural networks (DNNs) based methods \cite{o2022neural, richards2022control, Neural_Networks, saviolo2023learning}.} Multiple simulation and flight tests (indoor and outdoor) are implemented to verify the proposed scheme.


\subsection{Notation} \label{appen_notation}
Throughout the paper, $\mathbb{R}$ denotes the real number set; $\mathbb{Z}^+$ denotes the non-negative integer set; $\left| x\right|$ denotes the absolute value of a scalar $x$; $\| \bm{x} \|$ denotes the 2-norm of a vector $\bm{x}$; $\| \bm{A} \|_F$ denotes the \textit{Frobenius}-norm of a matrix $\bm{A}$; $\bm{x}_i$ denotes the $i$-th element of a vector $\bm{x}$; $\bm{X}_{ij}$ denotes the $i$-th row, $j$-th column element of a matrix $\bm{X}$; $\bm{X}\left( {i,:} \right)$ denotes the $i$-th row vector of a matrix $\bm{X}$; $\vec{\cdot}$ denotes a unit right shift operator; $\lambda_m\left(\cdot\right)$ represents the minimum eigenvalue of a matrix; $\textbf{I}$ and \textbf{0} represent the identity and zero matrices with appropriate sizes, respectively. Moreover, Mean Absolute Error (MAE) is defined as $MAE  =  {\frac{1}{n_d}\sum\nolimits_{i = 1}^{n_d} {\|{\bm{x}_i} - {\bm{x}_{d,i}}\|} }$ to evaluate test results, where ${n_d}$ denotes the size of collected data, $\bm{x}_i$ and $\bm{x}_{d,i}$ denote $i$-th evaluated variable and its desired value, respectively.

\section{Problem Formulation} \label{Problem_Formulation}
Consider a general control affine system of the form
\begin{align} \label{system_model}
	{\bm{\dot x}} = \bm{f}_x\left( {\bm{x}}\right)+ \bm{f}_u\left( {\bm{x}}\right)\bm{u}+ {\bm{\Delta}}\left( {\bm{x}, \bm{d}}\right),
\end{align}
where ${\bm{x}} \subset \mathcal{X} \in \mathbb{R}{^n}$ and ${\bm{u}} \subset \mathcal{U} \in \mathbb{R} {^o}$ denote the state and the control input, respectively; $\mathcal{X}$ and $\mathcal{U}$ are the state and control spaces of dimensionality $n$ and $o$, respectively; \textcolor{black}{$\mathcal{X}$ includes the unstable set;} ${\bm{f}_{\bm{x}}}\left(  \cdot  \right): {\mathbb{R}^n \rightarrow \mathbb{R}^n}$ and ${\bm{f}}_{\bm{u}}\left(  \cdot  \right){:\mathbb{R}^n \rightarrow \mathbb{R}^{n\times o}}$ are nonlinear mappings, which are continuously differentiable; ${\bm{\Delta}\left(  \cdot  \right): \mathbb{R}^n \times \mathbb{R}^m \rightarrow \mathbb{R}^n}$ represents the coupled disturbance, which is analytic. \textcolor{black}{${\bm{\Delta}}$ depends on known system state $\bm{x}$ and unknown external disturbance $\bm{d} \subset \mathcal{D} \in \mathbb{R}^m$} with $\mathcal{D}$ being the disturbance space of dimensionality $m$. Supplementary Materials \ref{appen_notation} provides the used notations.

\textbf{{Problem Statement}}:
Consider system \eqref{system_model}. The objective is to develop an algorithm to accurately estimate the coupled disturbance ${\bm{\Delta}}\left( {\bm{x}, \bm{d}}\right)$ only using control input $\bm{u}$ and measurable state $\bm{x}$.

Previous works \cite{o2022neural, richards2022control, huang2023datt, dawson2023safe, lutter2023combining, banderchuk2023combining} estimate the coupled disturbance $\bm{\Delta}$ with a bounded or bounded derivative assumption, i.e., there exists an unknown positive value ${\gamma} \in \mathbb{R}$ such that
\begin{align} \label{dis_ass0}
	\left\| {\bm{{\Delta}}}\left( {\bm{x}, \bm{d}}\right) \right\| \le \gamma \quad \text{or} \quad \left\| {\bm{\dot{\Delta}}}\left( {\bm{x}, \bm{d}}\right) \right\| \le \gamma.
\end{align}

Three limitations exist in the assumption. \textcolor{black}{\textbf{(L-1)} Logical paradox: the bounded assumption on $\bm{\Delta}$ actually requires the system state $\bm{x}$ to be preemptively bounded before proving the stability of the system \cite{10100656}. \textbf{(L-2)} Dynamic change: the evolution of $\bm{x}$ will change after the estimated disturbance is compensated. There is no guarantee that this assumption based on previous experience will always be satisfied. Consider a scenario where $\bm{x}$ constitutes the denominator of ${\bm{{\Delta}}}$. \textbf{(L-3)} Estimation conservativeness: the final disturbance estimation error usually depends on ${\gamma}$, which may be large.} \\
\textbf{Our Solution}: The core idea is to 1) decompose the coupled disturbance ${\bm{{\Delta}}}\left( {\bm{x}, \bm{d}}\right)$ into an unknown parameter matrix, a $\bm{x}$-related matrix, and a $\bm{d}$-related vector, with an arbitrarily small residual, 2) offline learn the unknown parameter from historical data, and 3) online estimate the remaining $\bm{d}$-related portion convergently. The whole process is schematized in Fig. \ref{overall_framework} and Algorithm \ref{alg:Framwork}. By resorting to the proposed algorithm, the limitations \textbf{(L-1)}-\textbf{(L-3)} are expected to be circumvented.

\begin{algorithm}[htb]   
	\caption{ Estimating coupled disturbance via the learning-based observer.}   
	\label{alg:Framwork}   
	\begin{algorithmic}[0] 
		\REQUIRE ~~\\ 
		Collect training dataset $\mathcal{D}_{tra} =\left\{{t_{f}^\mathfrak{n}, \bm{x}^\mathfrak{n} , \bm{\Delta}^\mathfrak{n}} \right\}$; set parameter $\delta$;\\
		Solve the RLS problem \eqref{minimization1};\\  
		\RETURN Output parameter matrix $\bm{\Theta}$;  
		\ENSURE ~~\\ 
		Set observer gain $\bm{\Lambda}_h\left(\bm{x}\right)$;
		\WHILE{running}
		\STATE 	Sample current feedback $\left\{\bm{x},\bm{u}\right\}$;
		\STATE	Solve state observer gain ${\bm{\Gamma}_h}$ for time-varying system \eqref{disturbance_system};
		\STATE	Run polynomial disturbance observer \eqref{DO_higher};
		\RETURN Estimated coupled disturbance $\bm{\hat{\Delta}}$; 
		\ENDWHILE
	\end{algorithmic}  
\end{algorithm}  

\section{Method} \label{Method}
\subsection{Decomposition of the coupled disturbance} \label{section_separation}

Before introducing the variable \textcolor{black}{decomposition} theorem for the coupled disturbance $\bm{\Delta}\left({\bm{x}, \bm{d}}\right)$, a preliminary lemma from \cite[Theorem 3]{o2022neural} is reviewed for a scalar $\bm{\Delta}_i\left({\bm{x}, \bm{d}}\right)$ firstly. 

\begin{lem} \label{lemma1}
	\cite[Theorem 3]{o2022neural} Assume an analytic function $\bm{\Delta}_i\left({\bm{x}, \bm{d}}\right): \mathbb{R}^n \times \mathbb{R}^m \rightarrow \mathbb{R}$ for all $\left[\bm{x},\bm{d}\right] \in \mathcal{X} \times \mathcal{D}$. For any small value $\epsilon > 0$, there always exist $p = O\left(\frac{{\log \left( {1/\varepsilon } \right)}}{{\sqrt {n + m} }}\right) \in \mathbb{Z}^+$, $\bm{\phi}_i\left(\bm{x}\right): \mathbb{R}^n \rightarrow \mathbb{R}^{1\times s}$ consisting of \textit{ Chebyshev} polynomials and unknown constant parameters, $\bm{\xi}\left(\bm{d}\right) \in \mathbb{R}^{s}$ consisting only of \textit{ Chebyshev} polynomials such that
	\begin{align}
		\mathop {\sup }\limits_{[x,d] \in {\mathcal{X} \times \mathcal{D}}} \left| \bm{\Delta}_i\left({\bm{x}, \bm{d}}\right)  - \bm{\phi}_i\left(\bm{x}\right) \bm{\xi}\left(\bm{d}\right)\right| \le \epsilon,
	\end{align}
	and $s = \left(p+1\right)^m = O\left(log\left(1/\epsilon\right)^m\right)$.
\end{lem}

\textcolor{black}{Note that Lemma \ref{lemma1} is initially build on  $\left[\bm{\bar{x}},\bm{\bar{d}}\right] \in \left[-1, 1\right]^n \times \left[-1, 1\right]^m$ in \cite[Theorem 3]{o2022neural}. By normalization, its result can be generalized to $\left[\bm{x},\bm{d}\right] \in \mathcal{X} \times \mathcal{D}$.} Lemma \ref{lemma1} concludes that the analytic coupled disturbance can be decoupled to a $\bm{x}$-related portion and a $\bm{d}$-related portion with an arbitrarily small residual. The $\bm{x}$-related portion remains unchanged in operation. Intuitively, it will be helpful to estimate the $\bm{d}$-related portion if the knowledge of $\bm{x}$-related portion can be exploited beforehand. In \cite{o2022neural, richards2022control}, DNNs are adopted to learn the $\bm{x}$-related portion, which needs laborious offline training and lacks interpretability. A more lightweight learning strategy is pursued here. To achieve that, we need to exploit Lemma \ref{lemma1} to drive a more explicit decomposition form.

\begin{thm} \label{thm_separation}
	$\bm{\Delta}_i\left({\bm{x}, \bm{d}}\right)$ is a function satisfying the assumptions in Lemma \ref{lemma1}, for all $i \in \left[1, 2, \cdots, n\right]$. For any small value $\epsilon' > 0$, there always exist $s_1\in \mathbb{Z}^+$; $s_2\in \mathbb{Z}^+$; an unknown constant parameter matrix $\bm{\Theta} \in \mathbb{R}^{n \times s_1}$, two functions $\bm{\mathcal{B}}\left(\bm{x}\right):  \mathbb{R}^n \rightarrow \mathbb{R}^{s_1 \times s_2}$ and $\bm{\xi}\left(\bm{d}\right):  \mathbb{R}^m \rightarrow \mathbb{R}^{s_2}$ that both consist only of \textit{Chebyshev} polynomials such that
	
	\begin{align} \label{eq_thm_separation}
		\mathop {\sup }\limits_{[x,d] \in {\mathcal{X} \times \mathcal{D}}} \left\| {\bm{\Delta}}\left( {\bm{x}, \bm{d}}\right) - \bm{\Theta}\bm{\mathcal{B}}\left(\bm{x}\right)\bm{\xi}\left(\bm{d}\right) \right\| \le \epsilon',
	\end{align}
	where $s_1 = \left(p+1\right)^{m+n} = O\left(log\left(\sqrt{n}/\epsilon'\right)^{m+n}\right)$ and $s_2 = \left(p+1\right)^m =  O\left(log\left(\sqrt{n}/\epsilon'\right)^m\right)$.
\end{thm}
\begin{proof}
	See Supplementary Materials \ref{appen_thm_separation}.
\end{proof}

\begin{rem}
	\textcolor{black}{Theorem \ref{thm_separation} extends the result of Lemma \ref{lemma1} to the multidimensional case and obtains a more explicit decomposed structure. It is proven that all unknown constant parameters of the coupled disturbance can be gathered into a matrix, which enables the coupled disturbance to be learned in a more explainable way compared with DNNs-based methods \cite{o2022neural, richards2022control, Neural_Networks}.}    
\end{rem}

{In comparison with the existing approaches \cite{o2022neural, richards2022control, huang2023datt, dawson2023safe, lutter2023combining, banderchuk2023combining} requiring bounded or bounded derivative assumption, the only premise of our framework is that $\bm{d}(t)$ is analytic. The premise can be readily fulfilled by many disturbances in real-word. Therefore, a theoretical foundation for the reliable implementation of an explicit learning process is achieved.} Subsequently, the following corollary is given.

\begin{cor} \label{Corollary_seperation}
	$\bm{\Delta}_i\left({\bm{x}, \bm{d}(t)}\right)$ is a function satisfying the assumptions in Lemma \ref{lemma1}, for all $i \in \left[1, 2, \cdots, n\right]$. Assume $\bm{d}(t)$ is analytic with respect to $t$. For any small value $\epsilon' > 0$, there always exist $s_1\in \mathbb{Z}^+$; $s_2\in \mathbb{Z}^+$; an unknown constant parameter matrix $\bm{\Theta} \in \mathbb{R}^{n \times s_1}$, two functions $\bm{\mathcal{B}}\left(\bm{x}\right):  \mathbb{R}^n \rightarrow \mathbb{R}^{s_1 \times s_2}$ and $\bm{\xi}\left(t\right):  \mathbb{R} \rightarrow \mathbb{R}^{s_2}$ that consist only of \textit{Chebyshev} polynomials such that
	\begin{align} \label{eq_cor_separation}
		\mathop {\sup }\limits_{[x,t] \in {\mathcal{X} \times \mathcal{D}}} \left\| {\bm{\Delta}}\left( {\bm{x}, \bm{d}(t)}\right) - \bm{\Theta}\bm{\mathcal{B}}\left(\bm{x}\right)\bm{\xi}\left({t}\right) \right\| \le \epsilon',
	\end{align}
	with $s_1 = \left(p+1\right)^{n+1} = O\left(log\left(\sqrt{n}/\epsilon'\right)^{n+1}\right)$, $s_2 = p+1 =  O\left(log\left(\sqrt{n}/\epsilon'\right)\right)$, and $\bm{\xi}\left(t \right) = {\left[ {\begin{array}{*{20}{c}}
				{{T_0}\left( t \right)}&{{T_1}\left( t \right)}& \cdots &{{T_p}\left( t \right)}
		\end{array}} \right]^T}$, where $T_{{i}}$ represents the $i$-th order \textit{Chebyshev} polynominal. 
\end{cor}
\begin{proof}
	The proof procedure is similar to Theorem \ref{thm_separation} by replacing the argument $\bm{d}$ with $t$.
\end{proof}

\begin{rem} \label{rem2}
	\textcolor{black}{Based on Theorem \ref{thm_separation}, Corollary \ref{Corollary_seperation} further resorts the unsamplable external disturbance $\bm{d}$ to the samplable feature $t$, which allows the unknown parameter matrix to be learned in a supervised way. The external disturbance $\bm{d}$ is ultimately decomposed into parameters in $\bm{\Theta}$ and $t$-related structure $\bm{\xi}\left({t}\right)$. A practical example is provided in Supplementary Materials \ref{appen_coro} to instantiate Corollary \ref{Corollary_seperation}.}
\end{rem}




\subsection{Learning the parameter matrix}
\label{online_learning}
In this part, an RLS optimization algorithm is established to learn parameter matrix $\bm{\Theta}$. Construct the training dataset $\mathcal{D}_{tra} =\left\{\left( {t_{f}^\mathfrak{n}, \bm{x}^\mathfrak{n} , \bm{\Delta}^\mathfrak{n} } \right)\ |\ \mathfrak{n}  = 1,2,\cdots,N \right\} $ with $N$ samples. $t_{f}$ represents the time in the offline training dataset. Note that $\bm{\Delta}^\mathfrak{n}$ can be calculated by using $\bm{\Delta}^\mathfrak{n} = {\bm{\dot x}}^\mathfrak{n} - \bm{f}_x\left( {\bm{x}}^\mathfrak{n}\right) - \bm{f}_u\left( {\bm{x}}^\mathfrak{n}\right)\bm{u}^\mathfrak{n}$, where ${\bm{\dot x}}^\mathfrak{n}$ can be obtained by offline high-order polynomial fitting. 

\textcolor{black}{Denote ${\bm{\Theta}^*}$ as the optimal solution under the dataset $\mathcal{D}_{tra}$.} The learning objective is formalized as
\begin{align} \label{minimization1}
	{\bm{\Theta} ^ * } = \arg \mathop {\min}\limits_{\textcolor{black}{\bm{\Theta}}} \frac{1}{2}\left[\sum\limits_{\mathfrak{n} = 1}^N {{{\left\| {\bm{\Delta}^\mathfrak{n} - \bm{\Theta}\bm{\mathcal{B}}_\mathfrak{n}\bm{\xi}_\mathfrak{n}} \right\|}^2}} +\delta\left\|\bm{\Theta}\right\|^2_F\right], 
\end{align}
where $\bm{\mathcal{B}}_\mathfrak{n} := \bm{\mathcal{B}}(\bm{x}^\mathfrak{n})$, $\bm{\xi}_\mathfrak{n} := \bm{\xi}\left({t}_f^\mathfrak{n}\right)$, and $\delta$ regularizes $\bm{\Theta}$.
{Furthermore, the closed-form solution of \eqref{minimization1} is given as:}
\begin{align} \label{OLDO_closed_form}
	{\bm{\Theta} ^ * } = \sum\limits_{\mathfrak{n}= 1}^N\bm{\Delta}^\mathfrak{n}\bm{\xi}_\mathfrak{n}^T\bm{\mathcal{B}}_\mathfrak{n}^T\cdot (\sum\limits_{\mathfrak{n}= 1}^N \bm{\mathcal{B}}_\mathfrak{n}\bm{\xi}_\mathfrak{n}\bm{\xi}_\mathfrak{n}^T\bm{\mathcal{B}}_\mathfrak{n}^T + \delta\bm{I})^{-1}.
\end{align}

Until now, the $\bm{x}$-related portion of ${\bm{\Delta}}\left( {\bm{x}, \bm{d}}\right)$ has been separated and the unknown constant parameters $\bm{\Theta}$ can be learned from the historical data. 

Denote $t_{l}$ as the online time. Note that the offline learning phase and online estimating {process} are in different time domains. In other words, the offset between $t_{f}$ and  $t_{l}$ is unknown. Therefore $\bm{\Delta}$ cannot be directly obtained by $\bm{\Theta}\bm{\mathcal{B}}\left(\bm{x}\right)\bm{\xi}\left({t}\right)$ online. Moreover, update $\bm{\xi}\left({t}\right)$ using feedback state can also improve the generalization to other disturbances. The remaining {issue} is to estimate $\bm{\xi}\left({t}\right)$ online from control input $\bm{u}$ and measured $\bm{x}$. By resorting to the polynomial disturbance observer to be designed, $\bm{\xi}\left({t}\right)$ can be exponentially estimated. 

\subsection{Estimation via a polynomial disturbance observer}

Before proceeding, \textcolor{black}{$\bm{\xi}\left({t}\right)$ in Corollary \ref{Corollary_seperation} can be further decomposed due to the structure of \textit{Chebyshev} polynomial (Eq. \eqref{xi})}. It can be rendered that
\begin{align} \label{further_decompose}
	\bm{\xi}\left({t}\right) = \mathcal{H}  \bm{\varsigma}\left({t}\right),
\end{align}
where $\mathcal{H} \in \mathbb{R}{^{s_2\times s_2}}$, 
\[\mathcal{H}\left( {i,:} \right) = \left\{ {\begin{array}{*{20}{c}}
		{\left[ {\begin{array}{*{20}{c}}
					1&0& \cdots &0
			\end{array}} \right],\ i = 1,}\\
		{\left[ {\begin{array}{*{20}{c}}
					0&1& \cdots &0
			\end{array}} \right],\ i = 2,}\\
		{2\vec {\bm{H}}\left( {i-1,:} \right) - \bm{H}\left( {i - 2,:} \right),\ 2 < i \le {s_2},}
\end{array}} \right.\]
$\bm{H}\left( {i,:} \right)$ denotes the $i$-th row vector of $\bm{H}$, $\vec{\cdot}$ denotes a unit right shift operator, and $\bm{\varsigma}\left( t \right)$ consists of polynomial basis functions, i.e., 
\begin{align}
	\bm{\varsigma} \left( t \right) = {[\begin{array}{*{20}{c}}
			1&t& \cdots &{{t^p}}
		\end{array}]^T} \in \mathbb{R}{^{{s_2}}} .
\end{align}

From \eqref{eq_cor_separation} and \eqref{further_decompose}, the coupled disturbance $\bm{\Delta} \left( \bm{x},\bm{d} \right)$ is finally represented as 
\begin{align} \label{disturbance_system}
	\left\{ {\begin{array}{*{20}{c}}
			{ \bm{\dot{\varsigma}} \left( t \right) = \mathcal{A}\bm{\varsigma} \left( t \right)},\\
			{\bm{\Delta}  =\bm{\Theta}\bm{\mathcal{B}}\left(\bm{x}\right)\mathcal{H}\bm{\varsigma} \left( t \right)},
	\end{array}} \right.
\end{align}
where $\mathcal{A}\in \mathbb{R}{^{s_2\times s_2}}$,
\[\mathcal{A}\left( {i,j} \right) = \left\{ {\begin{array}{*{20}{c}}
		{j,\ i = j + 1,}\\
		{0,\ i \ne j + 1.}
\end{array}} \right.\]

Define ${\bm{\hat{{\varsigma}}}}$ and ${\bm{{\tilde{\varsigma}}}}$ as the estimation of ${\bm{{{\varsigma}}}\left(t\right)}$ and the estimation error ${\bm{{\tilde{\varsigma}}}} = {\bm{{{\varsigma}}}\left(t\right)} - {\bm{\hat{{\varsigma}}}}$, respectively. The expected objective of the subsequent disturbance observer is to achieve
\begin{align} \label{object_higher_DO}
	{\bm{\dot{\tilde{\varsigma}}}} = \bm{\Lambda}_h\left(\bm{x}\right){\bm{\tilde{\varsigma}}},
\end{align}
where ${\bm{\Lambda}_h\left(\bm{x}\right)}: {\mathbb{R}^{n}} \rightarrow {\mathbb{R}^{s_2 \times s_2}}$ denotes the observer gain matrix which is designed to ensure the error dynamics \eqref{object_higher_DO} exponentially stable. To achieve \eqref{object_higher_DO}, the observer is designed as
\begin{align} \label{DO_higher}
	\begin{cases}
		\displaystyle {{\bm{\dot z}_h} = \mathcal{A} \bm{\hat {\varsigma}}  - {\bm{\Gamma}_h}({\bm{f}_x}\left( \bm{x} \right) + {\bm{f}_u}\left( \bm{x} \right)\bm{u}+\bm{\Theta}\bm{\mathcal{B}}\left(\bm{x}\right)\mathcal{H}\bm{\hat {\varsigma}})},\\
		\displaystyle {\bm{\hat {\varsigma}} = {\bm{z}_h} + {\bm{\Gamma} _h}\bm{x}},\\
		\displaystyle \bm{\hat{\Delta}} =  \bm{\Theta}\bm{\mathcal{B}}\left(\bm{x}\right)\mathcal{H}\bm{\hat {\varsigma}},
	\end{cases}
\end{align}
where ${\bm{{z}_h}}\in \mathbb{R} {^{s_2}}$ is an auxiliary variable and ${\bm{\Gamma}_h}\in {\mathbb{R}^{s_2 \times n}}$ is designed such that $\bm{\Lambda}_h\left(\bm{x}\right) = \mathcal{A}-{\bm{\Gamma}_h}\bm{\Theta}\bm{\mathcal{B}}\left(\bm{x}\right)\mathcal{H}$.

\begin{thm} \label{thm_DO_high}
	Consider the nonlinear system \eqref{system_model}. Under the designed polynomial disturbance observer \eqref{DO_higher}, the estimation error $\bm{\tilde{{\Delta}}}$ will converge to zero exponentially if ${\bm{\Gamma}_h}$ can be chosen to make the error dynamics
	\begin{align} \label{stable_condition}
		{\bm {\dot{\tilde {\varsigma}}}}= (\mathcal{A}-{\bm{\Gamma}_h}\bm{\Theta}\bm{\mathcal{B}}\left(\bm{x}\right)\mathcal{H}) {\bm {{\tilde {\varsigma}}}}
	\end{align}
	exponentially stable.
\end{thm}
\begin{proof}
	{{See Supplementary Materials}} \ref{appen_thm_DO_high}.
\end{proof}

According to \eqref{stable_condition}, the design of ${\bm{\Gamma}_h}$ for polynomial disturbance observer \eqref{DO_higher} is equivalent to the design of the state observer gain for the disturbance system \eqref{disturbance_system}. The methods of state observer design for linear time-varying systems have been developed in many previous works, such as the least-squares-based observer \cite{788544}, the extended linear observer \cite{busawon1999state}, and the block-input/block-output model-based observer \cite{325132}. Due to the convenience of adjusting the pole, the method proposed in \cite{busawon1999state} is employed here to decide ${\bm{\Gamma}_h}$ online. Moreover, the computational burden is mainly concentrated on the inverse of an observable matrix, which is suitable for practical systems. {A detailed discussion can refer to Supplementary Materials.}
\begin{rem}
	The polynomial disturbance observer \eqref{DO_higher} utilizes the offline learned knownledge $\bm{\Theta}\bm{\mathcal{B}}\left(\bm{x}\right)$ and \textcolor{black}{real-time $\left\{\bm{u}, \bm{x}\right\}$} to estimate $\bm{\xi}\left({t}\right)$ online. Two advantages can be achieved. One is that the offset between offline time $t_{f}$ and online time $t_{l}$ need not be known. The other is that the generalization to other kinds of external disturbance $\bm{d}$ can be improved according to real-time feedback (Section \ref{drone_test}).
\end{rem}

\section{Empirical Study} \label{Experiments}

\subsection{Learning performance on nonlinear functions}
A supervised learning strategy is synthesized in Section \ref{online_learning} to extract $\bm{\Theta}$ from disturbance $\bm{\Delta}$. In this part, the learning performance is exemplified and analyzed by following nonlinear functions
\begin{align}
	\label{nonlinear_functions}
	\left\{ {\begin{array}{*{20}{c}}
			{\bm{\Delta} \left( {\bm{x},\bm{d}} \right) = \sin \left(\bm{x}\right)\sin \left( \bm{d} \right),\ \bm{d}= t},\\
			{\bm{\Delta} \left( {\bm{x},\bm{d}} \right) = \bm{x} - \frac{1}{{12}}{\bm{x}^3} - \frac{1}{4}\bm{d},\ \bm{d} = {t^2}},\\
			{\bm{\Delta} \left( {\bm{x},\bm{d}} \right) =  - \frac{1}{9}\sin \left( \bm{x}\right)\bm{d},\ \bm{d} = {t^3}}.
	\end{array}} \right.
\end{align}

\begin{figure*}
	\centering
	\includegraphics[scale=0.18,trim=0 0 0 0,clip]{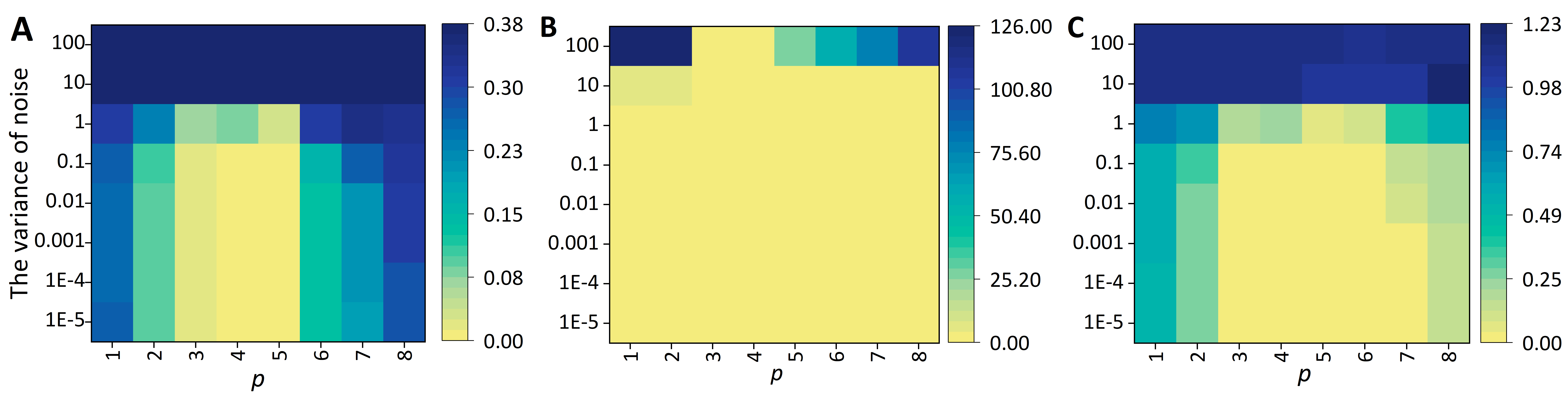}
	\caption{Learning errors (MAE) of three nonlinear functions \eqref{nonlinear_functions} under different noise variances $\bm{\sigma}_x^2$ and parameters $p$ on the test dataset.} 
	\label{Fig_Learning_performance}
\end{figure*}

\paragraph{Setup} The learning dataset is constructed from $\bm{x}\in \left[-2, 2\right]$ and $t \in \left[0, 4\right]$. $10000$ samples are collected and scrambled, where $5000$ for training and $5000$ for testing. The hyperparameter $\delta$ used for training is set as $0.01$. Moreover, the influences of measurement noise and the selection of $p$ (Corollary \ref{Corollary_seperation}) are also analyzed. The state $x$ in the training dataset is corrupted by noise $\mathcal{N}\left( {\bm{0},\bm{\sigma}_x^2} \right)$. The learning performance under different $\bm{\sigma}_x^2$ and $p$ is tested.    

\paragraph{Results} The learning errors of three nonlinear functions \eqref{nonlinear_functions} under different $\bm{\sigma}_x^2$ and $p$ are presented in Fig. \ref{Fig_Learning_performance}\textbf{A}-\textbf{C}. Two phenomena can be observed. On the one hand, the learning performance degrades as the noise variance increases. On the other hand, proper $p$ can achieve decent performance, since small and large $p$ can lead to underfitting and overfitting problems, respectively.

\subsection{Estimation performance on a second-order dynamics}
In this part, the estimation performance of the proposed polynomial disturbance observer is demonstrated through a toy example for intuitive understanding. Considering  a second-order \textit{Newton} system perturbated by a coupled disturbance

\begin{align} 
	\left\{ \begin{array}{l}
		{{\dot{\eta }} = {v}}, {{\dot{v}} = {a}},\\
		{m{{a}} = {u} + {\Delta}(v, d)},
	\end{array} \right.
\end{align}
with position ${\eta}\in \mathbb{R}$, velocity ${v}\in \mathbb{R}$, acceleration ${a}\in \mathbb{R}$, mass $m\in \mathbb{R}$, control input ${u} \in \mathbb{R}$, and coupled disturbance ${\Delta}(v, d)\in \mathbb{R}$. Note that the measured ${v}$ is corrupted by noise $\mathcal{N}\left( {{0},{\sigma} _v^2} \right)$. The coupled disturbance used in the simulation is modeled as
\begin{align} \label{coupled_dis_simulation}
	\Delta (v,d(t)) =  - {v^2} + 50 - 10t - 0.5{t^2}. 
\end{align}
The truth value of $\bm{\Theta}$ of \eqref{coupled_dis_simulation} can be derived, i.e., $\bm{\Theta} = \left[49.75, 0, -0.5, -10, 0, 0, 0.25, 0, 0\right]$.

The objective is to design control input ${u}$ so as to ensure that ${\eta}$ tracks the desired state ${\eta_d}$. The baseline controller adopts the proportional-derivative control, and the estimated disturbance ${\hat{\Delta}}$ by the proposed polynomial disturbance observer is compensated via feedforward, i.e., 
\begin{align} \label{simulation_controller}
	{u} = {K}_\eta {e}_\eta + {K}_v {e}_v - {\hat{\Delta}},
\end{align}
with gains ${K}_\eta\in \mathbb{R}$ and ${K}_v\in \mathbb{R} $, and tracking errors ${e}_\eta = {\eta_d}-{\eta}$ and ${e}_v = {v_d}-{v}$.

\paragraph{Setup}
The learning dataset is constructed from ${v}\in \left[-10, 10\right]$ and $t \in \left[0, 100\right]$. $10000$ samples are collected. The hyperparameter $\delta$ used for training is set as $0.01$. $p$ in Theorem \ref{thm_separation} is chosen as $2$. The desired tracking trajectory is set as $\eta_d = sin(\frac{1}{2}t)$. The variance ${\sigma} _v^2$ of imposed noise is set as $0.1$. The baseline controller gains ${K}_\eta$ and ${K}_v$ are tuned as $10$ and $25$, respectively. The NDO \cite{NDO_bounded_variation}, ESO \cite{ESO, GU2022105158}, $\mathcal{L}_1$ adaptive \cite{huang2023datt, hanover2021performance}, EVOLVER \cite{10288520}, and the baseline controller (without compensation) are taken as comparisons. For the sake of fairness, the observer gains (in charge of the convergence speed) of the compared NDO, $\mathcal{L}_1$ adaptive, EVOLVER, and the proposed one are set to be the same. All eigenvalues of $\bm{\Lambda}_h\left(\bm{x}\right)$ are set as $0.4$. 

\paragraph{Results}
Denote $\bm{\hat{\Theta}}$ as the learning result of $\bm{\Theta}$. By employing the proposed learning strategy designed in Section \ref{online_learning}, the learning error $\left\|\bm{\Theta}-\bm{\hat{\Theta}}\right\|$ is $1.1473\times10^{-3}$, which demonstrates the effectiveness of the proposed learning strategy in Section \ref{online_learning}. As shown in Fig. \ref{simulation_result}\textbf{A}, the tracking performance of learning-based methods (EVOLVER and proposed one \eqref{DO_higher}) outperform feedback-based methods (Baseline, NDO \cite{NDO_bounded_variation}, $\mathcal{L}_1$ adaptive \cite{huang2023datt, hanover2021performance}, ESO \cite{GU2022105158, ESO}). Focusing on feedback-based methods, there are always estimations lags from Fig. \ref{simulation_result}\textbf{B}. Compared with EVOLVER, a \textit{Koopman}-based online learning framework, the proposed one can avoid the dataset construction process and allows for a more in-depth learning of the coupled structure on disturbance \eqref{coupled_dis_simulation}, leading to the smallest tracking and estimation error (Figs. \ref{simulation_result}\textbf{A}-\textbf{C}). 


\begin{figure*}
	\centering
	\includegraphics[scale=0.35,trim=0 0 0 0,clip]{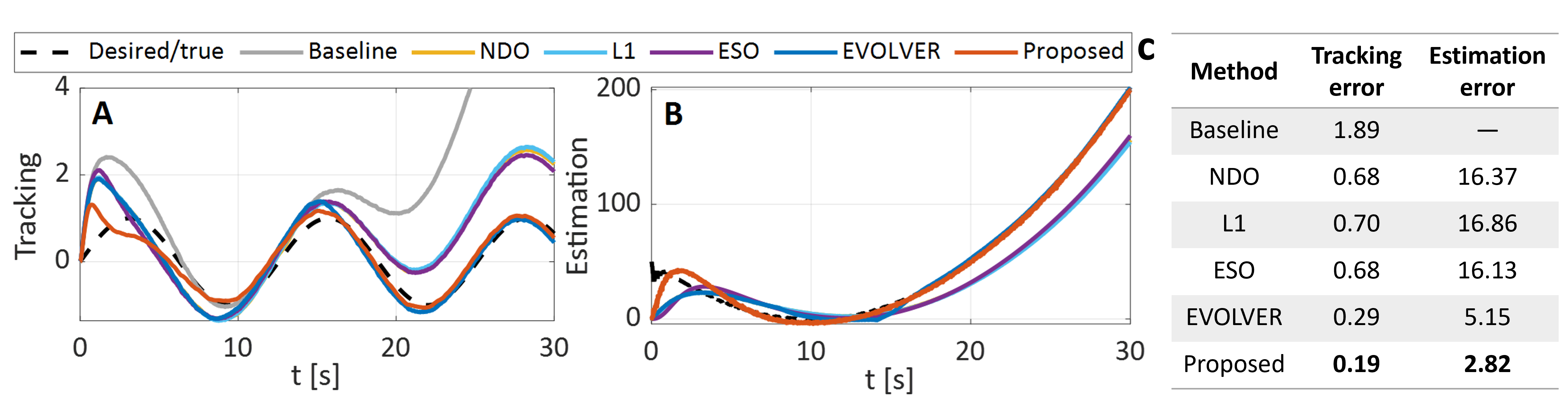}
	\caption{The tracking and estimation performance of the second-order dynamics example. \textbf{A}, The trajectory tracking performance. \textbf{B}, The disturbance estimation performance. \textbf{C}, MAEs of all compared methods: NDO \cite{NDO_bounded_variation}, ESO \cite{GU2022105158, ESO}, $\mathcal{L}_1$ adaptive \cite{huang2023datt, hanover2021performance}, EVOLVER \cite{10288520}, the baseline controller (without compensation) and the proposed one \eqref{DO_higher}.}
	\label{simulation_result}
\end{figure*}

\subsection{Control performance on a flying drone} \label{drone_test}

The proposed learning-based observer is further verified in real flight tests, which consider a drone flying under multiple disturbances. The conventional cascaded baseline control framework is employed here, where the coupled disturbances are further estimated by the proposed polynomial disturbance observer \eqref{DO_higher} to enhance the control accuracy.

\begin{figure}
	\centering
	\includegraphics[scale=0.27,trim=0 10 0 0,clip]{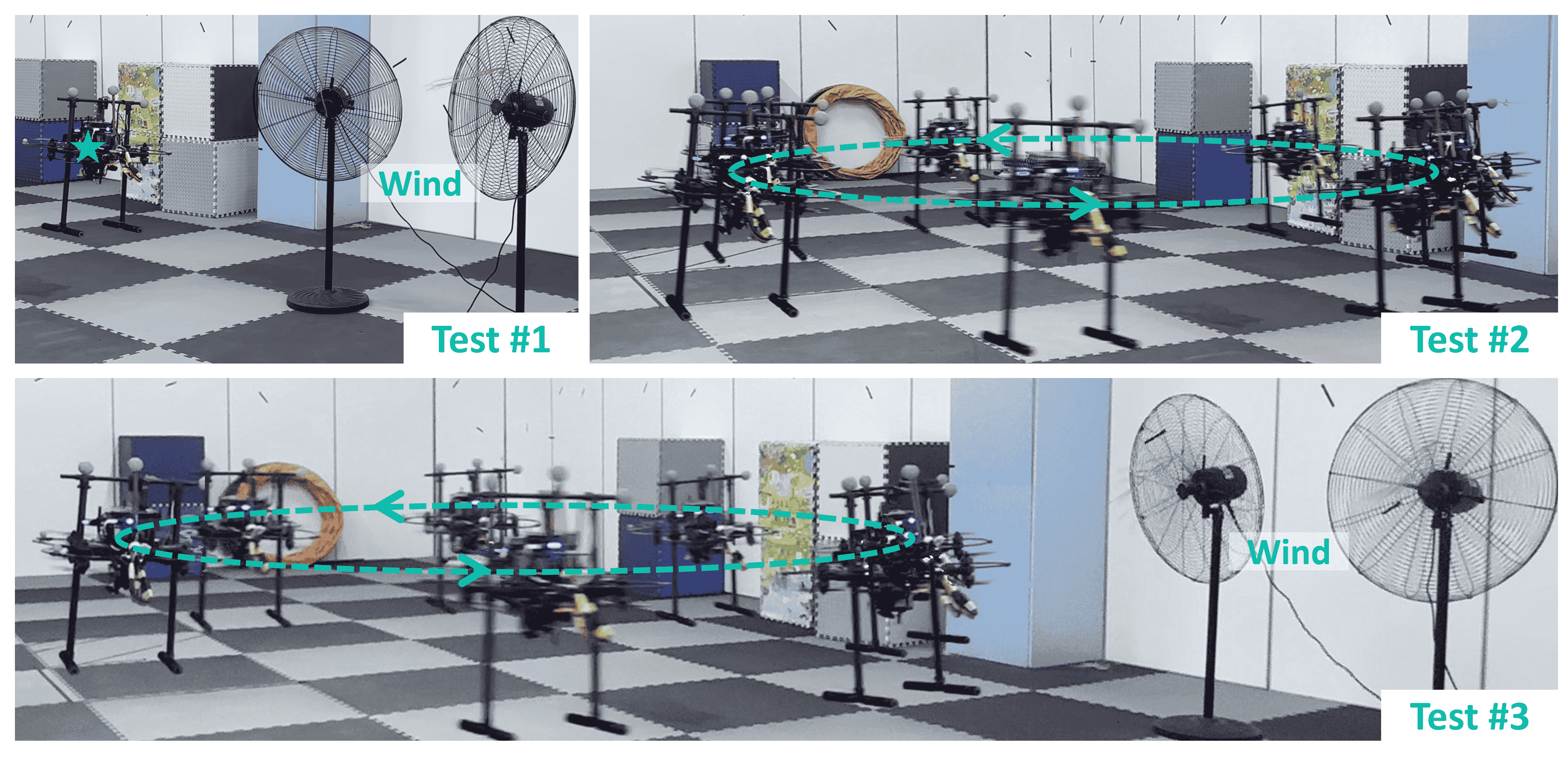}
	\caption{\textcolor{black}{Eexperimental scenarios in the drone example. The location state of the drone is provided by the NOKOV motion capture system. \textbf{A}, Hovering under wind disturbance. \textbf{B}, Circle flight under model uncertainty. \textbf{C}, Circle flight under model uncertainty and wind disturbance.}}
	\label{arrange_test}
\end{figure}


\paragraph{Setup}
Three scenarios are implemented indoor, including hovering under wind disturbance (Fig. \ref{arrange_test}\textbf{A}), circle flight under model uncertainty (Fig. \ref{arrange_test}\textbf{B}), and circle flight under both model uncertainty and wind disturbance (Fig. \ref{arrange_test}\textbf{C}). $\bm{\Theta}$ is offline trained under single mass uncertainty with $2500$ samples. $p$ in Theorem \ref{thm_separation} is chosen as $2$. The hyperparameter $\delta$ used for training is set as $0.01$.  Since the mass uncertainty mainly depends on unmeasurable acceleration, inspired by the time-delay embedding theory \cite{takens2006detecting}, $9$ historical velocities are utilized and reduced to $4$ through an autoencoder, which is finally regarded as $\bm{x}$ of ${\bm{\Delta}}\left( {\bm{x}, \bm{d}}\right)$. The mass uncertainty in tests is around $1\ kg$, constituting more than $35\%$ of the entire system. The wind is considered to verify the generalization ability. The maximum wing speed can reach at $5\ m/s$. The radius and period of the commended circle are set as $1\ m$ and $7\ s$, respectively. The observer \eqref{DO_higher} is running on an onboard Intel NUC with $100\ HZ$.

In the outdoor test, the drone is commanded to follow a circular trajectory in the presence of both mass uncertainty (around $1.2\ kg$) and natural wind, as illustrated in Fig. \ref{Outdoor_test}\textbf{E}. The radius and period of the commanded circle are set to be $1.2\ m$ and $7\ s$, respectively. The observer gain ${\bm{\Gamma}_h}$ in Eq. \eqref{DO_higher} is decided online \cite{busawon1999state} so that the eigenvalues of $\bm{\Lambda}_h\left(\bm{x}\right)$ are all $-0.6$.

\begin{figure*}
	\centering
	\includegraphics[width=0.9\linewidth]{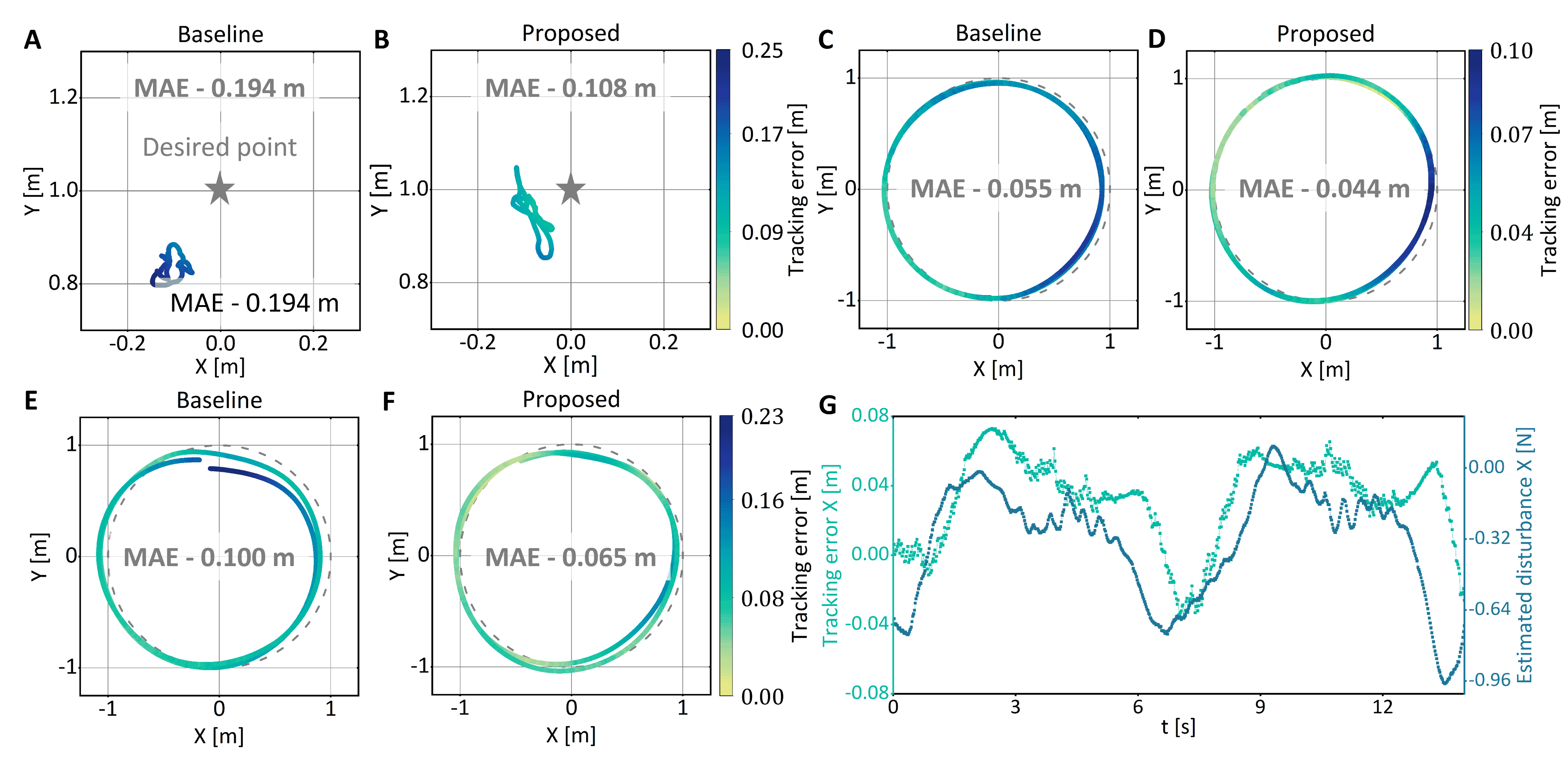}
	\caption{\textcolor{black}{The tracking and estimation performance in the drone example.} \textbf{A}-\textbf{B}, Test \#1. \textbf{C}-\textbf{D}, Test \#2. \textbf{E}-\textbf{F}, Test \#3. \textbf{G}, Tracking error and estimated disturbance along the $X$-axis in Test \#1.}
	\label{flight_result}
\end{figure*}

\paragraph{Results}
The control performance of all $3$ indoor tests is summarized in Fig. \ref{flight_result}. The control errors (MAE) are improved $44.3\%$, $20.0\%$, and $35.0\%$, respectively. The results show that the proposed one can successfully enhance the tracking performance with the compensation of estimated disturbance, even in the face of unseen wind disturbance. Fig. \ref{flight_result}\textbf{G} further provides the control and estimation performance along the $X$-axis in Test \#3, under both mass and wind disturbances. It can be seen the changing trends of tracking error and estimated disturbance are roughly consistent. Moreover, the computation burden is mainly around the solution of gain ${\bm{\Gamma}_h}$. {The computation cost is provided in Fig. \ref{Outdoor_test}\textbf{G}.}

We further conduct an outdoor flight test under natural wind. As depicted in Fig. \ref{Outdoor_test}\textbf{F}, the tracking delay and overshoot issues of the baseline method are mitigated by the proposed algorithm, particularly along the $X$-axis. Fig. \ref{Outdoor_test}\textbf{D} further presents the distributions of the tracking errors, which effectively demonstrates the robustness of the proposed approach in the outdoor scenario.

\begin{figure*}
	\centering
	\includegraphics[width=0.9\linewidth,trim=0 0 0 0,clip]{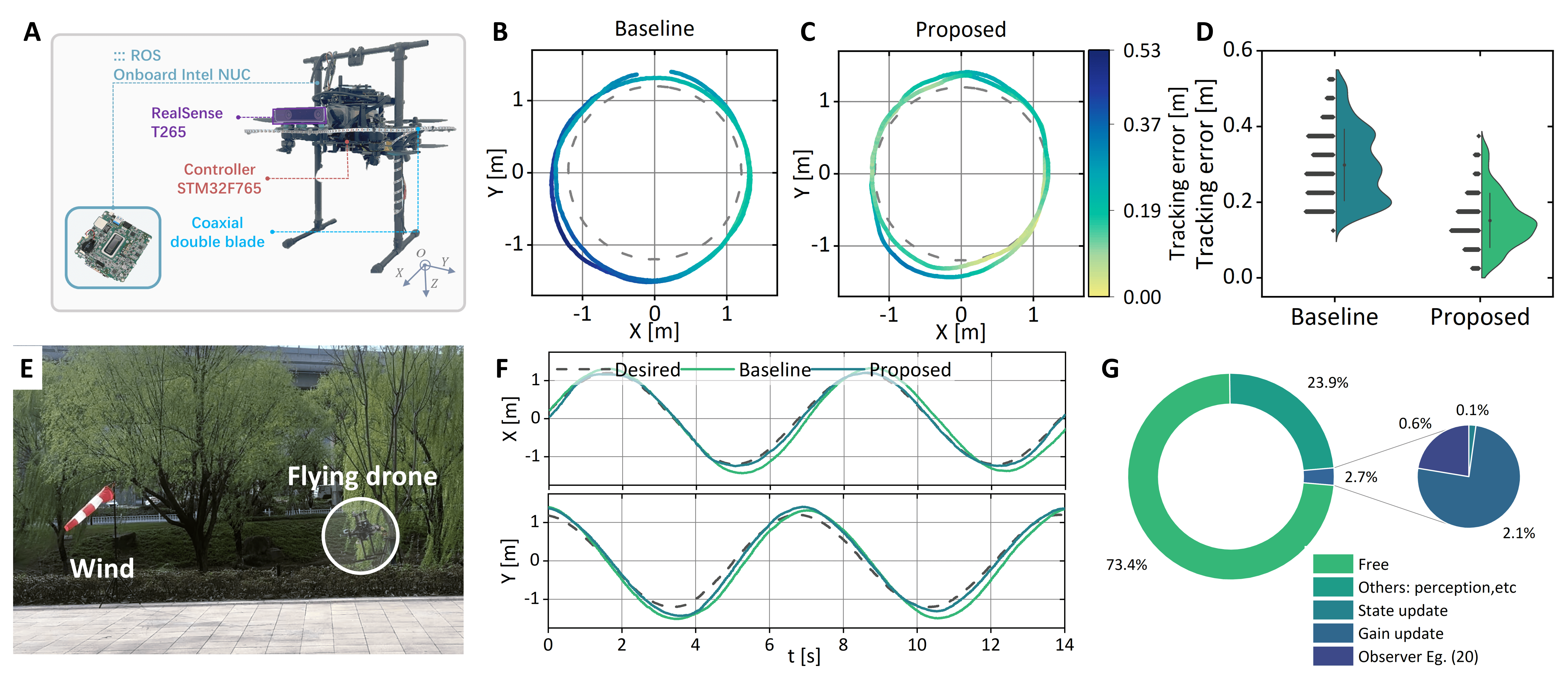}
	\caption{Flight results in the outdoor scenario. \textbf{A}, The employed flight platform. \textbf{B}, $2D$ view of the trajectory tracking by the baseline method. \textbf{C}, $2D$ view of the trajectory tracking by the proposed method. \textbf{D}, Violin plot of the tracking errors. \textbf{E}, The outdoor experimental scenario. \textbf{F}, Detailed tracking performance along $X$ and $Y$ directions of the compared two methods. \textbf{G}, The computation cost of the onboard Inter NUC.
	}
	\label{Outdoor_test}
\end{figure*}


\section{Conclusion} \label{Conclusion}
In this article, we propose a learning-based disturbance observer for robotic systems perpetuated by a coupled disturbance. The considered coupled disturbance is difficult to model explicitly. Firstly, a variable decomposition principle is presented by leveraging the \textit{Chebyshev} series expansion. An RLS-based learning strategy is subsequently developed to learn the separated unknown parameter matrix using historical data, which maintains a low computational complexity. Finally, the polynomial observer is developed by utilizing the learned structure, which can theoretically achieve a precise estimation of the coupled disturbance. The learning and estimation performance of the proposed method is verified by several simulational and experimental examples.

{Although the proposed framework showcases compelling theoretical and practical promise, the convergence of the presented uncertainty estimation algorithm is based on the analytic assumption on $\bm{d}$. Therefore, effectively handling the nonanalytic disturbances with a theoretical guarantee, such as contact-induced, step-like scenarios,  remains an open challenge that warrants further investigation.}




\bibliographystyle{IEEEtran}	
\bibliography{reference.bib} 

\clearpage

\section*{Supplementary Materials}

\renewcommand{\algorithmicrequire}{ \textbf{Offline learning phase:}} 
\renewcommand{\algorithmiccomment}{ \textbf{Initialize:}} 
\renewcommand{\algorithmicensure}{ \textbf{Online estimation phase:}} 

\subsection{Proof of Theorem \ref{thm_separation}} \label{appen_thm_separation}

\begin{proof}
	$\bm{\phi}_i\left(\bm{x}\right) \bm{\xi}\left(\bm{d}\right)$ in Lemma \ref{lemma1} is a compact product form of the truncated \textit{Chebyshev} expansions presented in \eqref{Chebyshev_series}. $b^i_{{k_1}, \cdots ,{k_n},{l_1}, \cdots ,{l_m}} \in \mathbb{R}$ represents the polynomial coefficient. Later in the article, $b^i_{{k_1}, \cdots ,{k_n},{l_1}, \cdots ,{l_m}}$ is  simplified as $b^i_{h_k,h_l}$ by letting ${h_k} = \sum\limits_{i = 1}^n {{k_i}{{\left( {p + 1} \right)}^{i - 1}}}$ and $h_l = \sum\limits_{i = 1}^m {{l_i}{{\left( {p + 1} \right)}^{i - 1}}}$. $T_{{i}}$ represents the $i$-th order \textit{Chebyshev} polynominal. 
	
	\begin{align} \label{Chebyshev_series}
		\bm{\phi}_i\left(\bm{x}\right) \bm{\xi}\left(\bm{d}\right) = &
		\sum\limits_{{k_1} = 0}^{p}  \cdots  \sum\limits_{{k_n} = 0}^{p} {\sum\limits_{{l_1} = 0}^{p}  \cdots  } \sum\limits_{{l_m} = 0}^{p} {b^i_{{k_1}, \cdots ,{k_n},{l_1}, \cdots ,{l_m}}} \notag \\
		& \cdot {T_{{k_1}}}\left( \bm{x}_1 \right) \cdots {T_{{k_n}}}\left(\bm{x}_n \right){T_{{l_1}}}\left( \bm{d}_1 \right) \cdots {T_{{l_m}}}\left( \bm{d}_m \right).
	\end{align}  
	
	\eqref{phi_i} and \eqref{xi} further detail the architectures of $\bm{\phi}_i\left(\bm{x}\right)$ and $\bm{\xi}\left(\bm{d}\right)$ in \eqref{Chebyshev_series} respectively, for the convenience of using in the remainder of this article.
	\begin{align} \label{phi_i}
		\bm{\phi}_i\left(\bm{x} \right) = [ & \sum\limits_{{k_1} = 0}^{p}  \cdots  \sum\limits_{{k_n} = 0}^{p} {b^i_{h_k,h_l}{T_{{k_1}}}\left(\bm{x_1}\right) \cdots {T_{{k_n}}}\left(\bm{x_n}\right){|_{h_l=0}}}\notag\\
		& \sum\limits_{{k_1} = 0}^{p}  \cdots  \sum\limits_{{k_n} = 0}^{p} {b^i_{h_k,h_l}{T_{{k_1}}}\left(\bm{x_1}\right) \cdots {T_{{k_n}}}\left(\bm{x_n}\right){|_{h_l=1}}}\notag\\
		&\quad \quad \quad \quad \quad \quad \quad \quad \quad \vdots \notag\\
		\sum\limits_{{k_1} = 0}^{p}  \cdots  &\sum\limits_{{k_n} = 0}^{p} {b^i_{h_k,h_l}{T_{{k_1}}}\left(\bm{x_1}\right) \cdots {T_{{k_n}}}\left(\bm{x_n}\right){|_{h_l=\left(p+1\right)^m -1}}}]^T.
	\end{align}
	
	\begin{align} \label{xi}
		\bm{\xi}\left(\bm{d} \right) = [ &{T_{{l_1}}}\left(\bm{d}_1\right) \cdots {T_{{l_m}}}\left(\bm{d}_m\right){|_{h_l=0}}\notag\\
		& {T_{{l_1}}}\left(\bm{d}_1\right) \cdots {T_{{l_m}}}\left(\bm{d}_m\right){|_{h_l=1}}\notag\\
		&\quad \quad \quad \quad \quad \quad \quad \quad \quad \vdots \notag\\
		& {T_{{l_1}}}\left(\bm{d}_1\right) \cdots {T_{{l_m}}}\left(\bm{d}_m\right){|_{h_l=\left(p+1\right)^m -1}}].
	\end{align}
	
	Denote the $j$-th column of $\bm{\phi}_i\left(\bm{x} \right)$ in \eqref{phi_i} as $\bm{\phi}_{ij}\left(\bm{x} \right)$. By further splitting $\bm{\phi}_{ij}\left(\bm{x} \right)$, it can be obtained that
	\begin{align} 
		\bm{\phi}_{ij}\left(\bm{x} \right) &=  \sum\limits_{{k_1} = 0}^{p}  \cdots  \sum\limits_{{k_n} = 0}^{p} {b^i_{h_k,h_l}{T_{{k_1}}}\left(\bm{x_1}\right) \cdots {T_{{k_n}}}\left(\bm{x_n}\right){|_{h_l=j-1}}}\notag\\
		& = \bm{b}^{ij} \cdot \bm{\Pi}\left(\bm{x}\right).
	\end{align}
	where $\bm{b}^{ij} = \left[{b_{{h_k},{h_l}}^i|_{{h_l} = j - 1}^{{h_k} = 0}}, {b_{{h_k},{h_l}}^i|_{{h_l} = j - 1}^{{h_k} = 1}}, \cdots,\right.$ $\left.{b_{{h_k},{h_l}}^i|_{{h_l} = j - 1}^{{h_k} = {{\left( {p + 1} \right)}^n} - 1}}\right] \in \mathbb{R}^{1\times \left(p+1\right)^n}$, and \[\bm{\Pi}\left(\bm{x}\right)  = \left[ {\begin{array}{*{20}{c}}
			{{T_{{k_1}}}\left(\bm{x_1}\right) \cdots {T_{{k_n}}}\left(\bm{x_n}\right){|_{{h_k} = 0}}}\\
			{{T_{{k_1}}}\left(\bm{x_1}\right) \cdots {T_{{k_n}}}\left(\bm{x_n}\right){|_{{h_k} = 1}}}\\
			\vdots \\
			{{T_{{k_1}}}\left(\bm{x_1}\right) \cdots {T_{{k_n}}}\left(\bm{x_n}\right){|_{{h_k} = {{\left( {p + 1} \right)}^n} - 1}}}
	\end{array}} \right] \in \mathbb{R}^{\left(p+1\right)^n}.\]
	
	Denote the $j$-th column of $\bm{\mathcal{B}}\left(\bm{x}\right)$ as $\bm{\mathcal{B}}_j\left(\bm{x}\right)$, and it is constructed as \[\bm{\mathcal{B}}_j\left( \bm{x} \right) = {\left[ {\underbrace {0, \cdots, 0}_{ {\left(j-1\right)\left(p+1\right)^n}},{\bm{\Pi}\left(\bm{x}\right) ^T},\underbrace {0, \cdots ,0}_{ {\left(\left(p+1\right)^{m+n}-j\left(p+1\right)^n\right)}}} \right]^T},\] 
	with $s_1 = \left(p+1\right)^{m+n}$ and $s_2 = \left(p+1\right)^{m}$.
	
	Denote the $i$-th row of $\bm{\Theta}$ as ${\bm{\Theta} _i}$, and it is constructed as 
	\[{\bm{\Theta} _i} = \left[ {{\bm{b}^{i1}},{\bm{b}^{i2}}, \cdots ,{\bm{b}^{i{{\left( {p + 1} \right)}^m}}}} \right] \in  \mathbb{R}^{1 \times {{\left( {p + 1} \right)}^{m + n}}}.\] It can be proven that
	\begin{align}
		\bm{\phi}_i\left(\bm{x}\right) = {\bm{\Theta} _i}\cdot \bm{\mathcal{B}}\left(\bm{x}\right).
	\end{align}
	Let $\bm{C}_i\left({\bm{x}, \bm{d}}\right)$ represent the $i$-th row of $\bm{C}\left({\bm{x}, \bm{d}}\right)\in  \mathbb{R}^n$ and define $\bm{C}_i\left({\bm{x}, \bm{d}}\right) = \bm{\phi}_i\left(\bm{x}\right)\bm{\xi}\left(\bm{d}\right)$, resulting in 
	\begin{align}
		\bm{C}\left({\bm{x}, \bm{d}}\right) = \bm{\phi}\left(\bm{x}\right)\bm{\xi}\left(\bm{d}\right)=\bm{\Theta}\bm{\mathcal{B}}\left(\bm{x}\right)\bm{\xi}\left(\bm{d}\right). \end{align}
	Set $\epsilon_i \le \left(\epsilon'/\sqrt{n}\right)$. From Lemma \ref{lemma1}, there exist $s_2^i = O\left(log\left(1/\epsilon\right)^m\right) \in \mathbb{Z}^+$  such that 
	\begin{align}
		\mathop {\sup }\limits_{[x,d] \in {{[ - 1,1]}^{n + m}}} \left| \bm{\Delta}_i\left({\bm{x}, \bm{d}}\right)  - \bm{C}_i\left({\bm{x}, \bm{d}}\right)\right| \le \epsilon_i.
	\end{align}
	Choose ${s_2} = \max \left\{ {s_2^1,s_2^2, \cdots s_2^n} \right\}$, it can be implies that
	\begin{align}
		\mathop {\sup }\limits_{[x,d] \in {{[ - 1,1]}^{n + m}}} \left\| \bm{\Delta}\left({\bm{x}, \bm{d}}\right)  - \bm{C}\left({\bm{x}, \bm{d}}\right)\right\| \le \sqrt {\sum\limits_{i = 1}^n {{\epsilon _i^2}}} \le \epsilon',
	\end{align}
	with $s_1 =  O\left(log\left(\sqrt{n}/\epsilon'\right)^{m+n}\right)$ and $s_2 =  O\left(log\left(\sqrt{n}/\epsilon'\right)^m\right)$.
\end{proof}

\subsection{Illustration on Corollary \ref{Corollary_seperation}} \label{appen_coro}
\textcolor{black}{A practical example is given here to instantiate the decomposition \eqref{eq_cor_separation}. In \cite{kai2017nonlinear, faessler2017differential}, the wind disturbance for the quadrotor is modeled as}
\begin{align} \label{wind_dis}
	\textcolor{black}{\bm{\Delta} = \bm{R}\bm{D}{\bm{R}^T}{\bm{v}_w},}
\end{align}
\textcolor{black}{where $\bm{D}\in\mathbb{R}^{3\times3}$ represents drag coefficients, $\bm{v}_w\in\mathbb{R}^{3}$ is the unknown external wind speed, and $\bm{R}\in\mathbb{R}^{3\times3}$ denotes the rotation matrix from body frame to inertial frame. By regarding $\bm{R}$ as $\bm{x}$ and $\bm{v}_w$ as $\bm{d}$ respectively, it can be seen that \eqref{wind_dis} has already been a decomposed form. However, the linear model \eqref{wind_dis} lacks accuracy as the high-order aerodynamics are not captured. By resorting to the decomposition \eqref{eq_cor_separation}, high-order aerodynamics can be included. Moreover, it is better to characterize the unknown time-varying wind speed as a polynomial function of $t$ with an appropriate order than to treat it as a constant or bounded value like in \cite{o2022neural,richards2022control}.}

\subsection{Proof of Theorem \ref{thm_DO_high}} \label{appen_thm_DO_high}

\begin{proof}
	From \eqref{DO_higher}, differentiate the estimation error ${\bm{{\tilde{\varsigma}}}}$. It can be implied that 
	\begin{align}
		{\bm {\dot{\tilde {\varsigma}}}} &= {\bm \dot{{{\varsigma}}}(t)} -{\bm \dot{{\hat {\varsigma}}}} \notag \\
		& = {\bm {\dot{{\varsigma}}}}(t) -  {\bm{\dot{z}}_h} - {\bm{\Gamma} _h}\bm{\dot{x}} \notag \\
		& { = } {\bm {\dot{{\varsigma}}}}(t) - \mathcal{A} \bm{\hat {\varsigma}}  + {\bm{\Gamma}_h}({\bm{f}_x}\left( \bm{x} \right) + {\bm{f}_u}\left( \bm{x} \right)\bm{u}+\bm{\Theta}\bm{\mathcal{B}}\left(\bm{x}\right)\mathcal{H}\bm{\hat {\varsigma}})- {\bm{\Gamma} _h}\bm{\dot{x}} \notag \\
		& {\mathop  = \limits^{(a)} } (\mathcal{A}-{\bm{\Gamma}_h}\bm{\Theta}\bm{\mathcal{B}}\left(\bm{x}\right)\mathcal{H}) {\bm {{\tilde {\varsigma}}}},
	\end{align}
	where $(a)$ is obtained by substituting \eqref{system_model}.
	
	It can be seen that the estimation error $\bm{\tilde{{\varsigma}}}$ will converge to zero exponentially if ${\bm{\Gamma}_h}$ can be chosen to make ${\bm {\dot{\tilde {\varsigma}}}}= (\mathcal{A}-{\bm{\Gamma}_h}\bm{\Theta}\bm{\mathcal{B}}\left(\bm{x}\right)\mathcal{H}) {\bm {{\tilde {\varsigma}}}}$ exponentially stable. Finally, the estimated coupled disturbance can converge to the truth value exponentially as $\bm{\tilde {\varsigma}} \to 0$ because of $\bm{\tilde{{\Delta}}} = \bm{\Theta}\bm{\mathcal{B}}\left(\bm{x}\right)\mathcal{H}\bm{\tilde {\varsigma}}$.
\end{proof}

\subsection{Existence of Observer Gain} \label{appen_abser}
The method proposed in \cite{busawon1999state} is employed here to decide ${\bm{\Gamma}_h}$ online. Three conditions need to be met to guarantee the existence of ${\bm{\Gamma}_h}$. \textbf{(C-1)} The observability matrix ${\bm M}_o(\bm{x})$ is full rank; \textbf{(C-2)} The system matrix $\mathcal{{A}}$, output matrix $\mathcal{C}(\bm{x}) = \bm{\Theta}\bm{\mathcal{B}}\left(\bm{x}\right)\mathcal{H}$, and their respective time derivatives are all bounded; \textbf{(C-3)} There exists a positive $\kappa$ such that ${\sup _{{\mu}  \ge 1}}\left\{ {\left\| {\bm{M}_{\mu} \dot{\bm{M}}_o(\bm{x}){\bm{M}_o^{ - 1}}(\bm{x}){\bm{M}_{\mu} ^{ - 1}}} \right\|} \right\} \le \kappa $, where
\begin{align}
	\bm{M}_0(\bm{x}) = \left[ {\begin{array}{*{20}{c}}
			{\mathcal{C}\left( \bm{x} \right)}\\
			{\mathcal{C}\left( \bm{x} \right)\mathcal{{A}}}\\
			\vdots \\
			{\mathcal{C}\left( \bm{x} \right){\mathcal{{A}}^{n - 1}}}
	\end{array}} \right], \quad \bm{M}_{\mu}  = \left[ {\begin{array}{*{20}{c}}
			{\frac{1}{\mu }}&{}&{}&{}\\
			{}&{\frac{1}{{{\mu ^2}}}}&{}&{}\\
			{}&{}& \ddots &{}\\
			{}&{}&{}&{\frac{1}{{{\mu ^n}}}}
	\end{array}} \right].
\end{align}
Due to the bounded RLS learning and differentiable property of \textit{Chebyshev} polynomial, \textbf{(C-2)} and \textbf{(C-3)} can be naturally met for the disturbance system \eqref{disturbance_system}. As for \textbf{(C-1)}, it mainly depends the trajectory of $\bm{x}$. Fig. \ref{eigen} shows the eigenvalue distributions of $(\mathcal{A}-{\bm{\Gamma}_h}\bm{\Theta}\bm{\mathcal{B}}\left(\bm{x}\right)\mathcal{H})$ in our simulational and experimental tests, it can be seen that the \textbf{(C-1)} is always satisfied (even in the hovering flight test). In other unexplored applications, the rank of ${\bm M}_o(\bm{x})$ can be detected in real time. No compensation will be given if ${\bm M}_o(\bm{x})$ is not full rank. 

Future work will attempt to incorperate online active leanring \cite{abraham2019active} into the presented observer, ensuring ${\bm M}_o(\bm{x})$ is full rank under explorative trajectories.

\renewcommand*{\thefigure}{S1}
\begin{figure}
	\centering
	\includegraphics[scale=0.16,trim=0 0 0 0,clip]{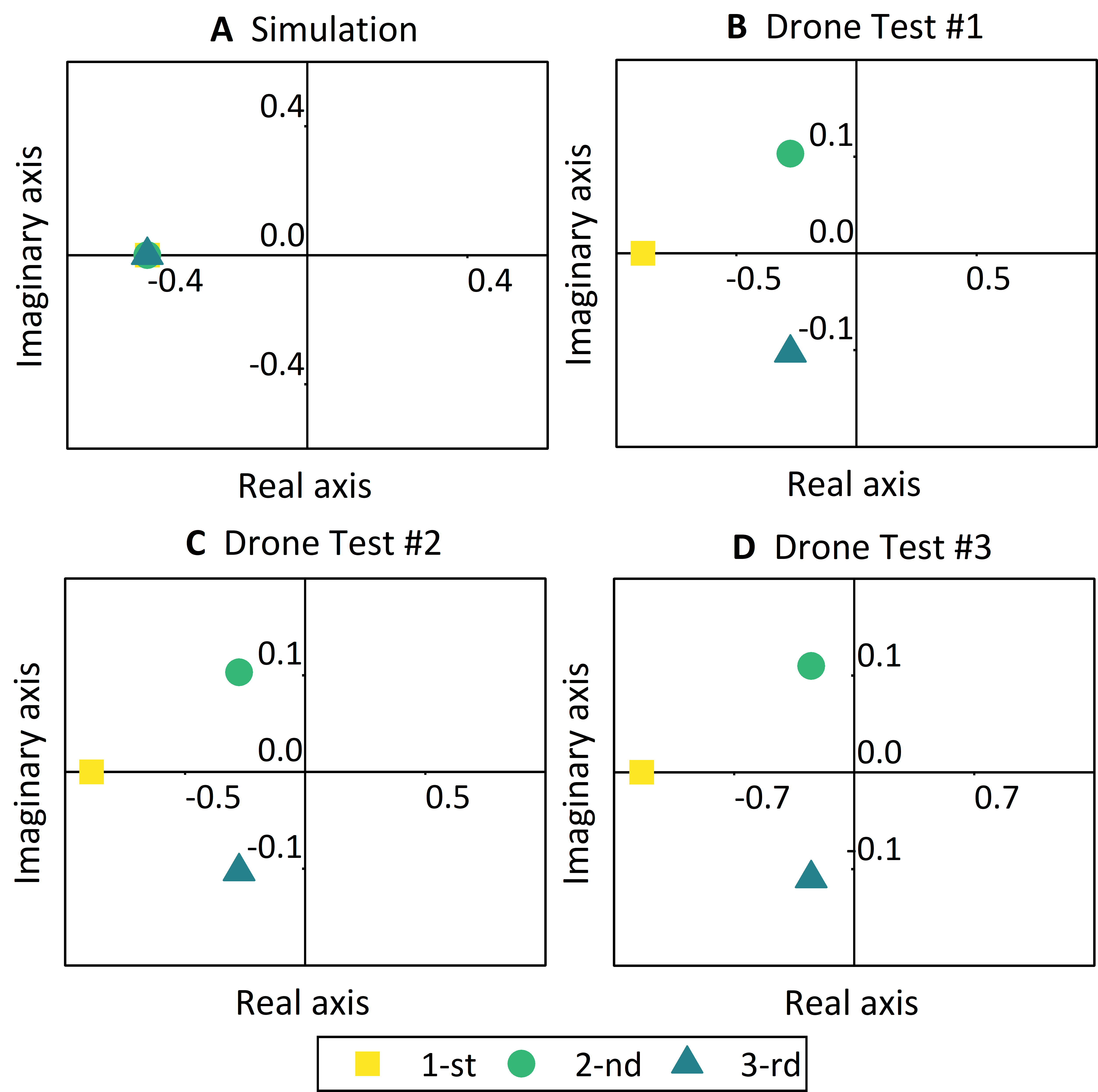}
	\caption{The eigenvalue distributions of $(\mathcal{A}-{\bm{\Gamma}_h}\bm{\Theta}\bm{\mathcal{B}}\left(\bm{x}\right)\mathcal{H})$ in our simulation and three drone tests (X-direction) during the entire runtime. It can be seen that all three eigenvalues in each test have negative real parts.}
	\label{eigen}
\end{figure}

\subsection{Drone Flight Test} \label{appen_exp_drone}

\subsubsection{Platform}

In the experiments, a coaxial double-blade drone equipped with an onboard Intel NUC is utilized, as shown in Fig. \ref{Outdoor_test}\textbf{A}. The baseline control algorithm runs on the STM32F765 microprocessor. Meanwhile, the disturbance observer operates on the Intel NUC. The control frequency is set at $100\ HZ$. For outdoor tests without a high-precision motion capture system, the onboard RealSense T265 is adopted for localization. Communication with the ground station is achieved through an ultra-wideband data transmission module, with a transmission frequency of $50\ HZ$.

\subsubsection{Computation cost}

Fig. \ref{Outdoor_test}\textbf{G} illustrates the computation cost of the utilized Intel NUC. During the flight test, $26\%$ of the CPU resources are occupied. The proposed learning-based disturbance observer only accounts for $2.7\%$ of the CPU resources. Among this $2.7\%$, $2\%$ is allocated for the update of the gain. Thus, the computational power consumption of the proposed algorithm is lightweight enough, enabling its application in realistic scenarios.


\end{document}